\documentclass[12pt]{article}
\usepackage{algorithm,rotating,verbatim,tabularx}
\usepackage[authoryear,round]{natbib}
\RequirePackage{hyperref}
\usepackage{amsmath,latexsym,amssymb,graphicx,amsthm}

\newtheorem{proposition}{Proposition}

\newtheorem{corollary}{Corollary}

\usepackage{setspace}
\DeclareMathOperator{\Z}{\mathbf{Z}}

\DeclareMathOperator{\M}{\mathbf{M}}

\DeclareMathOperator{\W}{\mathbf{W}}
\DeclareMathOperator{\V}{\mathbf{V}}
\DeclareMathOperator{\U}{\mathbf{U}}

\DeclareMathOperator{\uvec}{\mathbf{u}}
\DeclareMathOperator{\vvec}{\mathbf{v}}
\DeclareMathOperator{\zvec}{\mathbf{z}}
\DeclareMathOperator{\wvec}{\mathbf{w}}
\DeclareMathOperator{\D}{\mathbf{D}}

\DeclareMathOperator{\Pmat}{\mathbf{P}}

\DeclareMathOperator{\Q}{\mathbf{Q}}

\DeclareMathOperator{\R}{\mathbf{R}}
\DeclareMathOperator{\Y}{\mathbf{Y}}

\DeclareMathOperator{\X}{\mathbf{X}}

\DeclareMathOperator*{\argmin}{\mathrm{argmin}}
\DeclareMathOperator*{\maximize}{\mathrm{maximize}}

\begin{document}

\title{ \bf \Large Regularized Partial Least Squares
  with an Application to NMR Spectroscopy}

\author{Genevera I. Allen$^{1,2}$\footnotemark , Christine
  Peterson$^{2}$, Marina Vannucci$^{2}$ \\ \& Mirjana
  Maleti{\'c}-Savati{\'c}$^{1}$ \\
{\small $^{1}$ Department of Pediatrics-Neurology, Baylor College of
  Medicine } \\ {\small Jan and Dan Duncan Neurological Research Institute, Texas
  Children's Hospital,} \\ {\small $^{2}$ Department of Statistics, Rice
  University}\\ 
}
\footnotetext{To whom correspondence should be addressed; Department
  of Statistics, Rice University, MS 138, 
6100 Main St., Houston, TX 77005 (email:  gallen@rice.edu)}
\date{}

\maketitle
\thispagestyle{empty}

\begin{abstract}
High-dimensional data common in genomics, proteomics, and chemometrics
often contains complicated correlation structures.  Recently, partial
least squares (PLS) and Sparse PLS methods have gained attention in
these areas as dimension reduction techniques in the context of
supervised data analysis.
We introduce a framework for Regularized PLS by solving a
relaxation of the SIMPLS optimization problem with penalties on the
PLS loadings vectors.  Our approach enjoys many advantages including 
flexibility, general penalties, easy interpretation of results, and
fast computation in high-dimensional settings.   We also outline
extensions of our methods leading to novel methods for Non-negative
PLS and Generalized PLS, an adaption of PLS for structured data.  
We demonstrate the utility of our methods through simulations and a
case study on proton Nuclear Magnetic Resonance (NMR) spectroscopy
data.   
\end{abstract}


\section{Introduction}

Technologies to measure high-throughput biomedical data in proteomics,
chemometrics, and genomics have led to a proliferation of
high-dimensional data that pose many statistical challenges.  As
genes, proteins, and metabolites, are biologically interconnected, the
variables in these data sets are often highly correlated.  In this
context, several have recently advocated using partial least squares
(PLS) for dimension reduction of supervised data, or data with a
response or labels \citep{nguyen_2002,
  boulesteix_strimmer_2007, rossouw_2008, 
  chun_sparse_pls_2010}.  First introduced by 
\citet{wold_1966} as a 
regression method that uses least squares on a set of derived inputs
accounting for multi-colinearities, others have since proposed
alternative methods for PLS with multiple responses
\citep{de_jong_simpls_1993} and for
classification \citep{marx_pls_glm_1996, barker_pls_2003}.  More 
generally, PLS can be interpreted as a dimension reduction technique
that finds projections of the data that maximize the covariance
between the data and the response.  Recently, several have proposed to
encourage sparsity in these projections, or loadings vectors, to
select relevant features in high-dimensional data \citep{rossouw_2008, 
  chun_sparse_pls_2010}.  In this
paper, we seek a more general and flexible framework for regularizing
the PLS loadings that is computationally efficient for
high-dimensional data.

There are several motivations for regularizing the PLS loadings
vectors.  Partial least squares is closely related to principal
components analysis (PCA); namely, the PLS loadings can be computed by
solving a generalized eigenvalue problem \citep{de_jong_simpls_1993}.
Several have shown 
that the PCA projection vectors are asymptotically inconsistent in
high-dimensional settings \citep{johnstone_jasa_2009, jung_pca_2009}.
This is also the case for the PLS 
loadings, recently shown in \citet{nadler_2005} and
\citet{chun_sparse_pls_2010}.  For PCA, 
encouraging sparsity in the loadings has been shown to yield
consistent projections \citep{johnstone_jasa_2009, amini_2009}.  While
an analogous result has not 
yet been shown in the context of PLS, one could surmise that such a
result could be attained.  In fact, this is the motivation for
\citet{chun_sparse_pls_2010}'s recent Sparse PLS method.  In addition
to consistency 
motivations, sparsity has many other qualities to recommend it.
The PLS loadings vectors can be used as a data compression technique
when making future predictions; sparsity further compresses the data.
As many variables in high-dimensional data are noisy and irrelevant,
sparsity gives a method for automatic feature selection.  This leads
to results that are easier to interpret and visualize.

While sparsity in PLS is important for high-dimensional
data, there is also a need for more general and flexible regularized
methods.  Consider NMR spectroscopy as a motivating example.  This
high-throughput data measures the spectrum of chemical resonances
of all the latent metabolites, or small molecules, present in a biological
sample \citep{nicholson_2008}.  Typical experimental data consists of
discretized, 
functional, and  non-negative spectra with variables measuring in the
thousands for only a small number of samples.  Additionally, variables
in the spectra have complex dependencies arising from correlation at
adjacent chemical shifts, metabolites resonating at more than one
chemical shift, and overlapping resonances of latent metabolites
\citep{degraaf_nmr_2007}.   Because of
these complex dependencies, there is a long history of using PLS to reduce
the NMR spectrum for supervised data \citep{goodacre_2004,
  dunn_2005b}.  Classical PLS or 
Sparse PLS, however, are not optimal for this data as they
do not account for the non-negativity or functional nature of the
spectra.  In this paper, we seek a more flexible approach to 
regularizing PLS loadings that will permit (i) general penalties such
as to encourage sparsity, group sparsity, or smoothness, (ii)
constraints such  
as non-negativity, and (iii) directly account for known data
structures such as ordered chemical shifts for NMR spectroscopy.  Our
framework, based on a penalized relaxation of the SIMPLS optimization
problem \citep{de_jong_simpls_1993}, also leads to a more
computationally efficient numerical algorithm.

As we have mentioned, there has been previous work on penalizing the
PLS loadings.  For
functional data, \citet{goutis_smooth_pls_1996} and
\citet{reiss_fpls_2007} have extended PLS to encourage 
smoothness by adding smoothing penalties.  Our approach is
more closely 
related to the Sparse PLS methods of \citet{rossouw_2008} and
\citet{chun_sparse_pls_2010}.   In the latter, a generalized
eigenvalue problem related to PLS objectives is penalized to achieve
sparsity, although they solve an approximation to this problem via the
elastic net Sparse PCA approach of \citet{zou_sparse_pca_2006}.
Noting that PLS can be 
interpreted as performing PCA on the deflated cross-products matrix,
\citet{rossouw_2008} replace PCA with Sparse PCA using the approach of
\citet{shen_spca_2008}.  
We choose to adopt a more direct approach.  Our Sparse PLS
method, instead, penalizes a generalized SVD problem directly with an
$\ell_{1}$-norm penalty that is a concave relaxation of the SIMPLS
criterion; our method, then, is more closely related to the Sparse PCA
approaches of \citet{witten_pmd_2009} and \citet{allen_gmd_2011}.  We
will show that this more direct framework
has numerous advantages including
generalizations permitting various penalties that are norms,
non-negativity constraints, generalizations for structured data,
greater algorithmic flexibility, and fast computational 
approaches for high-dimensional data.

The paper is organized as follows.  Our framework for Regularized
Partial Least Squares (RPLS) is introduced in Section
\ref{section_rpls}.  In Section
\ref{section_ext}, we introduce two novel extensions of PLS and RPLS:
Non-negative PLS and Generalized PLS for structured data.  We
illustrate the comparative strengths of our approach in Sections
\ref{section_sims} and \ref{section_nmr} through
simulation studies and a case study on NMR 
spectroscopy data, respectively, and conclude with a discussion in
Section \ref{section_dis}.





\section{Regularized Partial Least Squares}
\label{section_rpls}


In this section, we introduce our framework for regularized partial
least squares.  While most think of PLS as a regression technique,
here we separate the steps of the PLS approach into the dimension
reduction stage where the PLS loadings and factors are computed and a
prediction stage where regression or classification using the PLS
factors as predictors is performed.  As our contributions lie in our
framework for regularizing the PLS loadings in the dimension reduction
stage, we focus on this in the first three subsections, and then
discuss considerations for regression and classification problems in
Section ~\ref{section_pred}.

\subsection{RPLS Optimization Problem}

Introducing notation, we observe data (predictors), $\X \in \Re^{n
  \times p}$, with $p$ variables measured on $n$ samples and a 
response $\Y \in \Re^{n \times q}$.  We will assume that the columns
  of $\X$ have been 
previously standardized.  The possibly multivariate response ($q>1$)
could be continuous as in regression or encoded by dummy variables
to indicate classes as in \citet{barker_pls_2003}, a consideration
which we ignore while developing our methodology.   The $p \times q$
sample cross-product matrix 
is denoted as $\M = \X^{T} \Y$.

Both of the two major algorithms for computing the multivariate PLS
factors, NIPALS \citep{wold_1966} and SIMPLS
\citep{de_jong_simpls_1993}, can be written as 
solving a single-factor eigenvalue problem of the
following form at each step: 
$\maximize_{\vvec} \ \ \vvec^{T} \M \M^{T} \vvec \ \ \textrm{subject to} \ \
\vvec^{T} \vvec = 1,$
where $\vvec \in \Re^{p}$ are the PLS loadings.
\citet{chun_sparse_pls_2010} extend this problem by adding an
$\ell_{1}$-norm constraint, $|| \vvec || \leq t$, to induce sparsity
and solve an approximation to this problem using the Sparse PCA method
of \citet{zou_sparse_pca_2006}. \citet{rossouw_2008} replace this
optimization problem with that of the Sparse PCA approach of
\citet{shen_spca_2008}.

We take a simpler and more direct approach.  Notice that the single
factor PLS problem can be re-written 
as the following:
$\maximize_{\vvec, \uvec} \ \ \vvec^{T} \M \uvec$ $ \textrm{subject to} \ \
\vvec^{T} \vvec = 1 \ \& \ \uvec^{T} \uvec = 1,$
where $\uvec \in \Re^{q}$ is a nuisance parameter.  The equivalence of
these problems was pointed out in
\citet{de_jong_simpls_1993} and is a well understood matrix analysis fact
\citep{horn_johnson}.  Our single-factor RPLS problem penalizes a
direct concave relaxation of this problem:
\begin{align}
\label{sing_fac_rpls}
\maximize_{\vvec, \uvec} \ \ \vvec^{T} \M \uvec - \lambda P( \vvec )
\ \ \textrm{subject to} \ \ 
\vvec^{T} \vvec \leq 1 \ \& \ \uvec^{T} \uvec = 1.
\end{align}
Here, we assume that $P()$ is a convex penalty function that is a norm
or semi-norm; these assumptions are discussed further in the
subsequent section.  To induce sparsity, for example, we can take $P(\vvec)
= || \vvec ||_{1}$.  Notice that we have relaxed the equality
constraint for $\vvec$ to an inequality constraint.  In doing so, we
arrive at an optimization problem that is simple to maximize via an
alternating strategy.  Fixing $\uvec$, the problem in $\vvec$ is
concave, and fixing $\vvec$ the problem is a quadratically
constrained linear program in $\uvec$ with a global solution.  Our
approach is most closely related to some recent direct bi-concave
relaxations for two-way penalized matrix factorizations
\citep{witten_pmd_2009, allen_gmd_2011}.  Studying the solution to
this problem and its properties in 
the subsequent section will reveal some of the major advantages of
this optimization approach.

Computing the multi-factor PLS solution via the two traditional multivariate
approaches, SIMPLS and NIPALS, require solving optimization problems
of the same form as the single-factor PLS problem at each step.   The
SIMPLS 
method is more direct and has several benefits within our framework;
thus, this is the approach we adopt.  The
algorithm begins by solving the single-factor PLS problem;
subsequent factors solve the single-factor problem for a Gram-Schmidt
deflated cross-products matrix.   If we
let the matrix of projection weights $\R_{k} \in \Re^{p \times k}$
be defined 
recursively then, $\R_{k} = [ \R_{k-1} \ \ \X^{T} \zvec_{k} / \zvec_{k}^{T}
  \zvec_{k} ]$ where $\zvec_{k} = \X \vvec_{k}$ is the $k^{th}$ sample
PLS factor.  The Gram-Schmidt projection 
matrix $\Pmat_{k} \in \Re^{p \times p}$ is given by $\Pmat_{k} =
\mathbf{I} -  \R_{k}  (\R_{k}^{T} \R_{k} )^{-1} \R_{k}$, which ensures
that $\vvec_{k}^{T} \X^{T} \X \vvec_{j} = 0$ for $j < k$.  
Then, the optimization
problem to find the $k^{th} $ SIMPLS loadings vector is the same as
the single-factor problem with the cross-products matrix, $\M$,
replaced by the deflated matrix, $\hat{\M}^{(k)} = \Pmat_{k-1}
\hat{\M}^{(k-1)}$ \citep{de_jong_simpls_1993}.  Thus, our multi-factor
RPLS replaces $\M$ in \eqref{sing_fac_rpls} with $\hat{\M}^{(k)}$ to
obtain the $k^{th}$ RPLS factor.  


The deflation approach employed via the NIPALS algorithm is not as
direct.  One typically defines
a deflated matrix of predictors and responses, $\tilde{\X}_{k} = \X (
\mathbf{I} - \V_{k} \R_{k}^{T} )$ and $\tilde{\Y}_{k} = \Y (
\mathbf{I} - \V_{k} \R_{k}^{T} )$, with the matrix of projection
weights 
defined as above, and then solves an eigenvalue problem in 
this deflated space: $\maximize_{\wvec_{k}} \ \wvec_{k}^{T}
\tilde{\X}_{k}^{T} \tilde{\Y}_{k} \tilde{\Y}_{k}^{T} \tilde{\X}_{k}
\wvec_{k} \ \textrm{subject to} \ \wvec_{k}^{T} \wvec_{k} = 1$
\citep{wold_1966}.  The 
PLS loadings in the original space are then recovered by $\V_{k} =
\W_{k} ( \R_{k} \W_{k})^{-1}$.  While one can incorporate
regularization into the loadings, $\wvec_{k}$ (as suggested in
\citet{chun_sparse_pls_2010}), this is not as desirable.  If one
estimates sparse deflated loadings, $\wvec$, then much of the sparsity
will be lost in the transform to obtain $\V$.   In fact, the elements
of $\V$ will be zero if and only if the corresponding 
elements of $\W$ are zero for all values of $k$.  Then, each of the
$K$ PLS loadings will have the exact same sparsity pattern, loosing
the flexibility of each set of loadings having adaptively different
levels of sparsity. Given this, the more direct deflation approach of
SIMPLS is our preferred framework.

\subsection{RPLS Solution}

A major motivation for our optimization framework for RPLS
 is that it leads to a simple and direct
solution and algorithm.  
Recall that the single-factor RPLS problem, \eqref{sing_fac_rpls}, is
concave in $\vvec$ with $\uvec$ fixed and is a quadratically
constrained linear program in $\uvec$ with $\vvec$ fixed.  Thus, we
propose to solve this problem by alternating maximizing with respect
to $\vvec$ and $\uvec$.  Each of these maximizations has a simple
analytical solution:
\begin{proposition}
\label{prop_rpls_sol}
Assume that $P()$ is convex and homogeneous of
order one, that is $P()$ is a norm or semi-norm.  Let $\uvec$ be fixed at
$\uvec'$  such that $\M  \uvec' \neq 0$ or $\vvec$ fixed at
$\vvec'$ such that $\M^{T} \vvec' \neq 0$. 
Then, the coordinate updates, $\uvec^{*}$ and $\vvec^{*}$, maximizing
the single-factor RPLS problem, \eqref{sing_fac_rpls}, are given by the
following:  Let $\hat{\vvec} = \argmin_{\vvec} \{
\frac{1}{2}|| \M \uvec' - \vvec ||^{2} - \lambda P(\vvec)\}$.  Then, 
$\vvec^{*} = \hat{\vvec} / || \hat{\vvec} ||_{2}$ if $|| \hat{\vvec}
||_{2} > 0$ and $\vvec^{*} = 0$ otherwise, and $  \uvec^{*} = 
  \M^{T} \vvec' / || \M^{T} \vvec' ||_{2} $. 
When these factors are updated iteratively, they monotonically
increase the objective and converge to a local optimum.
\end{proposition}

While the full proof of this result is given in the appendix, we note
that this follows closely the Sparse PCA approach of
\citet{witten_pmd_2009} and the use general penalties within PCA
problems of \citet{allen_gmd_2011}.  Our RPLS problem can then be
solved by a multiplicative update for $\uvec$ and by a simple
re-scaled penalized regression problem for $\vvec$.   
The assumption that $P()$ is a norm or semi-norm encompasses many
 penalties types including the $\ell_{1}$-norm or lasso
\citep{tibshirani_1996} and the  
$\ell_{1} / \ell_{2}$-norm or group lasso \citep{yuan_2006_group}, fused
lasso \cite{tibshirani_2005_fused}.
For many possible penalty types, there exists a simple solution to the
penalized regression problem.  With a lasso penalty, $P(\vvec) = ||
\vvec ||_{1}$, for example, the solution is given by
soft-thresholding: 
$\hat{\vvec} = S( \M \uvec, \lambda)$, where $S(x,\lambda) =
\mathrm{sign}(x) (| x| - \lambda)_{+}$ is the soft-thresholding
operator.  Our approach gives a more general framework for incorporating
regularization directly in the PLS loadings that yield
simple and computationally attractive solutions.

We note that the RPLS solution is guaranteed be at most a local
optimum of \eqref{sing_fac_rpls}, a result that is typical of other
penalized PCA problems \citep{zou_sparse_pca_2006, shen_spca_2008,
  witten_pmd_2009, lee_ssvd_2010, allen_gmd_2011} and sparse PLS
methods \citep{rossouw_2008, chun_sparse_pls_2010}.  For a special
case, however, our problem has a global solution:
\begin{corollary}
\label{cor_rpls_sol}
When $q=1$, that is when $\Y$ is univariate, then the global solution to the
single-factor penalized PLS problem \eqref{sing_fac_rpls} is given
by the following: Let $\hat{\vvec} = \argmin_{\vvec} \{
\frac{1}{2}|| \M  - \vvec ||^{2} - \lambda P(\vvec)\}$.  Then, 
$\vvec^{*} =  \hat{\vvec} / || \hat{\vvec} ||_{2}$ if $||
  \hat{\vvec} ||_{2} > 0$ and $\vvec^{*} = 0$ otherwise. 
\end{corollary}
This, then is an important advantage of our framework over competing
methods.

\subsection{RPLS Algorithm}

\begin{algorithm}[!!t]
\caption{$K$-Factor Regularized PLS}
\label{rpls_alg}
\begin{enumerate}
\item Center the columns of $\X$ and $\Y$.  Let $\hat{\M}^{(1)} = \X^{T} \Y$.
\item For $k = 1 \ldots K$:
\begin{enumerate}
\item Initialize $\uvec_{k}$ and $\vvec_{k}$ to the first left and
  right singular vectors of $\hat{\M}^{(k)}$.
\item Repeat until convergence:
\begin{enumerate}
\item Set $\uvec_{k} = \frac{  (\hat{\M}^{(k)})^{T}
  \vvec_{k}}{|| (\hat{\M}^{(k)})^{T} \vvec_{k}||_{2}}$.  
\item Set $\hat{\vvec}_{k} = \argmin_{\vvec'_{k}} \left\{ ||
  \hat{\M}^{(k)}  
  \uvec_{k} - \vvec'_{k} ||_{2}^{2} - \lambda_{k} P( \vvec'_{k} )
  \right\}$. 
\item Set $\vvec_{k} =  \hat{\vvec}_{k} / ||
  \hat{\vvec}_{k} ||_{2}$ if  $|| \hat{\vvec}_{k} ||_{2} > 0$, and set
  $\vvec_{k} = 0$ and exit the algorithm otherwise.
\end{enumerate}
\item RPLS Factor:  $\zvec_{k} = \X \vvec_{k}$.
\item RPLS projection matrix: Set $\R^{(k)} = [
  \R^{(k-1)} \ \ \X^{T} \zvec_{k} /  \zvec_{k}^{T}   \zvec_{k} ]$ and
  $\Pmat_{k} = \mathbf{I} -   \R^{(k)} ( (\R^{(k)})^{T}
  \R^{(k)} )^{-1} (\R^{(k)})^{T}$.  
\item Orthogonalization Step: $\hat{\M}^{(k+1)} =  \Pmat_{k} \hat{\M}^{(k)}$. 
\end{enumerate}
\item Return RPLS Factors $\zvec_{1} \ldots \zvec_{K}$ and RPLS
  Loadings: $\vvec_{1} \ldots \vvec_{K}$.
\end{enumerate}
\end{algorithm}

Given our RPLS optimization framework and solution, we now put these
together in the RPLS algorithm, Algorithm \ref{rpls_alg}.  Note that
this algorithm is a direct extension of the SIMPLS algorithm
\citep{de_jong_simpls_1993}, 
where the solution to our single-factor RPLS problem,
\eqref{sing_fac_rpls}, replaces the typical eigenvalue problem in Step
2 (b).  Since our RPLS problem is 
non-concave, there are potentially many local solutions and
thus the initializations of $\uvec$ and $\vvec$ are important.  Similar to
much of the Sparse PCA literature \citep{zou_sparse_pca_2006,
  shen_spca_2008}, we recommend initializing
these factors to the global single-factor SVD solution, Step 2 (a).
Second, notice that choice of the regularization parameter, $\lambda$,
is particularly important.  If $\lambda$ is large enough that
$\vvec_{k} = 0$, then the $k^{th}$ RPLS factor would be zero and the
algorithm would cease.  Thus, care is needed when selecting the
regularization parameters to ensure they remain within the relevant
range.  For the special case where $q = 1$, computing $\lambda_{max}^{(k)}$,
the value at which $\hat{\vvec}_{k} = 0$, is a straightforward
calculation following from the Karush-Khun-Tucker conditions.  With
the LASSO penalty, for example, this gives $\lambda_{max}^{(k)} =
\mathrm{max}_{i} | \hat{\M}^{(k)}_{i} |$ \citep{glmnet}.  For general $q$,
however, $\lambda_{max}$ does not have a closed form.  While one could
use numerical solvers to find this value, this is a needless
computational effort.  Instead, we recommend to perform the algorithm
over a range of $\lambda$ values, discarding any values resulting in a
degenerate solution from consideration.  Finally, unlike
deflation-based Sparse PCA methods which can exhibit poor behavior for
very sparse solutions, due to orthogonalization with respect to the
data, our RPLS loadings and factors are well behaved with
large regularization parameters.

Selecting the appropriate regularization
parameter, $\lambda$, is an important practical consideration.
Existing methods that incorporate regularization in PLS have 
suggested using cross-validation or other model selection methods in
the ultimate regression or classification stage of the full PLS
procedure \citep{reiss_fpls_2007, chun_sparse_pls_2010}.  While one
could certainly implement these 
approaches within our RPLS framework, we suggest a simpler and more
direct approach.  We select $\lambda$ within the dimension reduction
stage of RPLS, specifically in Step 2 (b) of our RPLS Algorithm.
Doing so, has a number of advantages.  First, this increases
flexibility as it separates selection of $\lambda$ from deciding how
many factors, $K$, to use in the prediction stage, permitting  a separate
regularization parameter, $\lambda_{k}$, to be selected for each RPLS
factor.  Second, coupling
selection of the regularization parameter to the prediction stage
requires fixing the supervised modeling method before computing the
RPLS factors.  With our approach, the RPLS factors can be computed and
stored to use as predictors in a variety of modeling procedures.  Finally,
separating selection of $\lambda_{k}$ and $K$ in the prediction stage
is computationally advantageous as a grid search over tuning parameters
is avoided.  Nesting selection of $\lambda$ within Step 2 (b) is also
faster as recent developments such as warm starts and active set
learning can be used to efficiently fit the entire path of solutions
for many penalty types \citep{glmnet}.  Practically, selecting $\lambda_{k}$
within the dimension reduction stage is analogous to selecting the
regularization parameters for Sparse PCA methods on $\hat{\M}^{(k)}$.
Many approaches including cross-validation \citep{shen_spca_2008,
   owen_2009_cv} and BIC methods
\citep{lee_ssvd_2010, allen_gmd_2011} have been suggested for this
purpose; in results given in 
this paper, we have implemented the BIC method as described in
\citet{allen_gmd_2011}.  Selection of the number of RPLS factors, $K$,
will largely 
be dependent on the supervised method used in the prediction stage,
although cross-validation can be used with an method.

Computationally, our algorithm is an efficient approach.  As discussed
in the previous section, the particular computational 
requirements for computing the RPLS loadings in Step 2 (b) are penalty
specific, but are minimal for a wide class of commonly used penalties.
Beyond Step 2 (b), the major computational requirement is 
inverting the weight matrix, $\R_{k}^{T} \R_{k}$, to compute the
projection matrix.  Since this matrix is found recursively via the
Gram-Schmidt scheme, however, employing properties of the Schur
complement can reduce the computational effort to that of matrix
multiplication $O(pk)$ \citep{horn_johnson}.  Finally, notice that we take
the RPLS factors to be the direct projection of the data by the RPLS
loadings.  
Overall, the advantages of our RPLS framework and algorithm include
(1) computational efficiency, (2) flexible modeling, and (3) direct
estimation of the RPLS loadings and factors.

\subsection{RPLS for Regression and Classification}
\label{section_pred}

While many think of PLS as a single approach to regression, we have
separated the dimension reduction stage from the prediction stage
where the PLS factors, $\Z$, replace the original predictors.
As many have advocated using PCA, or 
even supervised PCA \citep{bair_2006}, as a dimension reduction
technique prior supervised modeling, RPLS may be a powerful
alternative in this context.  While studying the behavior of our RPLS
method for particular supervised techniques is beyond the scope of
this paper, we outline here some considerations for using our
framework in common regression and classification problems.

Applying our RPLS framework in regression problems where $\Y$ encodes
the response, is straightforward.  For univariate responses, our
framework has the added benefit that each RPLS loadings vector is the
global solution to the underlying penalized optimization problem.
With traditional PLS regression, there is an interesting connection between
Krylov sequences and the PLS regression coefficients, namely the
latter are the minimum to a least squares problem constrained so that
the coefficients lie within the Krylov subspace spanned by $\{ \Y,
\X \X^{T} \Y, (\X \X^{T} )^{2} \Y , \ldots , ( \X \X^{T} )^{K} \Y
\}$ \citep{kramer_2007}.  As we take the RPLS factors to be a direct
projection of the 
RPLS loadings, this connection to Krylov sequences is broken, although
perhaps for prediction purposes, this is immaterial.

While in the context of regression, traditional approaches
such as cross-validation can be used to find the number of RPLS
factors, $K \leq n$, an approach suggested by
\citet{huang_pen_pls_2004} may have 
added benefits for large data sets.  They propose to post-select the
number of factors by adding a sparse penalty, minimizing the following
criterion: $|| \Y - \Z \beta ||_{2}^{2} + \gamma || \beta ||_{1}$.  As the
PLS or RPLS factors are orthogonal, there is a simple solution for the
coefficients that automatically selects the number of factors,
$\hat{\beta} = S( \Z^{T} \Y, \gamma )$.  In our simulation study in
Section \ref{section_sims}, we use this penalization approach to
automatically selecting the number of RPLS factors.  
Also we note that a recent paper
directly computes the degrees of freedom for PLS regression that can
be used for model selection with BIC and AIC methods
\citep{kramer_2011}.   As
the relationship between Krylov sequences and our RPLS factors no
longer holds, however, this approach cannot be directly employed with our
methods.


Many have suggested using the PLS factors for classification by coding
the response as dummy variables indicating the classes for discriminant
analysis \citep{barker_pls_2003} or by using the exponential family links for
generalized linear models \citep{marx_pls_glm_1996,
  chung_spls_class_2010}, approaches that can be used in conjunction
with RPLS.   Interestingly,
\citet{barker_pls_2003} have shown that coding the response with dummy
variables 
scaled according the the class size yields PLS loadings vectors that
are a scaled version of Fisher's discriminant vectors.  Thus, our RPLS
framework may lead to an alternative formulation for sparse or
regularized LDA, a connection which we leave to future work to
explore.  Finally, while again cross-validation approaches can be used
to select the number of RPLS factors for classification, it is common
to compute a number of discriminant vectors equal to the number of
classes.  For our case study in Section \ref{section_nmr}, this is the
approach we adopt.



\section{Extensions}
\label{section_ext}

As our framework for regularizing PLS is general, there are many
possible extensions of our methodology.  We focus here on two novel
extensions of 
PLS and RPLS that will be particularly useful for understanding
spectroscopy data.  These include generalizations for PLS and RPLS
with structured data and non-negative PLS and RPLS.

\subsection{Generalized PLS for Structured Data}

Recently, \citet{allen_gmd_2011} proposed a generalization of PCA
(GPCA) that
is a appropriate for high-dimensional structured data,  or data in which the variables are
associated with some known distance metric.  As motivation, consider
NMR spectroscopy data where variables are ordered on the spectrum and
variables at adjacent chemical shifts are known to be highly
correlated.  Classical multivariate techniques such as PCA and PLS
ignore these structures; GPCA
encodes structure into a matrix factorization problem
through positive semi-definite quadratic operators such as Laplacians
or kernel smoothers \citep{allen_gmd_2011, allen_snn_gpca_2011}.
Similar to GPCA, we seek to directly 
account for known structure in PLS and within our RPLS framework.

Let us define the quadratic operator, $\Q \in \Re^{p \times p}: \Q
\succeq 0$, that encodes the known structural relationships between
variables.  
By transforming all inner-product spaces to those induced 
by the $\Q$-norm, we can define our single-factor Generalized RPLS
optimization problem in the following manner:
\begin{align}
\label{g_rpls}
\maximize_{\vvec, \uvec} \ \  \vvec^{T} \Q  \M
\uvec - \lambda P( \vvec ) \ \ \textrm{subject to } \ \
\vvec^{T} \Q 
\vvec \leq 1, \  \& \  \uvec^{T} \uvec = 1.
\end{align}
For the multi-factor Generalized RPLS problem, the factors and
projection matrices are also 
changed.  The $k^{th}$ factor is given by $\zvec_{k} = \X \Q
\vvec_{k}$, the weighting matrix, $\R_{k} = [ \R_{k-1} \ \ \X^{T}
  \zvec_{k} / \zvec_{k}^{T} \zvec_{k} ]$ as before, and the
projection matrix is $\Pmat_{k} = \mathbf{I} - \R_{k}^{T} ( \R_{k}^{T}
\Q \R_{k} )^{-1} \R_{k}^{T}$. The deflated cross-products matrix is
then given by $\hat{\M}^{(k)} = \Pmat_{k-1} \Q \hat{\M}^{(k-1)}$.
Note that if $\lambda = 0$ and if the
inequality constraint is forced to be an equality constraint, then we
have the optimization problem for Generalized PLS.  Notice also that
instead of enforcing orthogonality of the PLS loadings with respect
to the data, $\vvec_{k}^{T} \X^{T} \X \vvec_{j}$, the Generalized PLS
problem enforces orthogonality in a projected data space,
$\vvec_{k}^{T} \Q \X^{T} \X \Q \vvec_{j}$.  If we let $\tilde{\Q}$ be a matrix
square root of $\Q$ as defined in \citet{allen_gmd_2011}, then
\eqref{g_rpls} is equivalent to the multi-factor RPLS problem for $\tilde{\X}
= \X \tilde{\Q}$ and $\tilde{\vvec} = \tilde{\Q} \vvec$.  This
equivalence is shown in the proof of the solution to \eqref{g_rpls}.

As with PLS and our RPLS framework, Generalized PLS and RPLS can be
solved by coordinate-wise updates that converge to the global and
local optimum respectively:
\begin{proposition}
\label{prop_g_rpls}
\begin{enumerate}
\item Generalized PLS: The Generalized PLS problem, \eqref{g_rpls} when
  $\lambda = 0$, is solved by the first set of GPCA factors of
  $\M$.  The global solution to the Generalized PLS
  problem can be found by iteratively updating the following until
  convergence: 
$\vvec =  \M \uvec  / ||  \M
  \uvec ||_{\Q}$ and $\uvec =  \M^{T} \Q
   \vvec / || \M^{T} \Q \vvec ||_{2}$, where $|| x ||_{\Q}$ is defined
  as $\sqrt{ x^{T} \Q x }$.  
\item Generalized RPLS: Under the assumptions of Proposition
  \ref{prop_rpls_sol}, let \\ $\hat{\vvec} = \argmin_{\vvec'} \{ ||
   \M \uvec - \vvec' ||_{\Q}^{2} + \lambda P( \vvec'
  ) \}$, then the coordinate-wise updates to \eqref{g_rpls} are given
  by: $\vvec^{*} = \hat{\vvec} / || \hat{\vvec} ||_{\Q}$ if
  $|| \hat{\vvec} ||_{\Q} > 0$ and $\vvec^{*} = 0$ otherwise, and
  with $\uvec^{*}$ defined as above.
When updated iteratively, these converge to a local optimum of
\eqref{g_rpls}.  
\end{enumerate}
\end{proposition}
Thus, the solution to our Generalized RPLS problem can be solved by a
generalized penalized least squares problem.  Algorithmically, solving
the multi-factor Generalized PLS and RPLS problems follow the same
structure as that of  Algorithm \ref{rpls_alg}.
The solutions outlined above replace Step 2 (b), with the altered
Generalized RPLS factors and projections matrices replacing Steps 2
(c), (d), and (e).  In other words, Generalized PLS or RPLS is
performed by finding the GPCA or Regularized GPCA factors of a
deflated cross-products matrix, where the deflation is performed to
rotate the cross-products matrix so that it is orthogonal to the data
in the $\Q$-norm.  Computationally, these algorithms can be performed
efficiently using the techniques described in \citet{allen_gmd_2011}
that do not require inversion or taking eigenvalue decompositions of
$\Q$.  Thus, the Generalized PLS and RPLS methods are
computationally feasible for high-dimensional data sets.


We have shown the most basic extension of GPCA technology to PLS and
our RPLS framework, but there are other possible formulations.   For
two-way 
data, projections in the ``sample'' 
space may be appropriate in addition to projecting variables in the
$\Q$-norm.  With neuroimaging data, for example, 
the data matrix may be oriented as brain locations, voxels, by time
points.  As the time series are most certainly not independent, one
may wish to transform these inner product spaces using another
quadratic operator, $\W \in \Re^{n \times n}$, changing $\M$ to
$\X^{T} \W \Y$ and $\R_{k}$ to $\R_{k} = [ \R_{k-1} \ \ \X^{T}
  \W \zvec_{k} / \zvec_{k}^{T} \W \zvec_{k} ]$, analogous to
\citet{allen_gmd_2011}.  Overall, we have outlined a novel extension
of PLS and our RPLS methodologies to work with high-dimensional
structured data.





\subsection{Non-Negative PLS}

Many have advocated estimating non-negative matrix factors
\citep{lee_nmf_1999} and 
non-negative principal component loadings
\citep{hoyer_sparse_nmf_2004} as a way to 
increase interpret-ability of multivariate methods.  For scientific
data sets such as NMR spectroscopy in which variables are naturally
non-negative, enforcing non-negativity of the loadings vectors can
greatly improve interpretability results and the performance of
methods \citep{allen_snn_gpca_2011}.  Here, we illustrate how to
incorporate non-negative 
loadings into our RPLS framework.   Consider the optimization problem
for single-factor Non-negative RPLS: 
\begin{align}
\label{nn_rpls}
\maximize_{\vvec, \uvec} \ \  \vvec^{T}  \M
\uvec - \lambda P( \vvec ) \ \ \textrm{subject to } \ \
\vvec^{T} 
\vvec \leq 1, \  \uvec^{T} \uvec = 1 \ \& \ \vvec \geq
0.
\end{align}
Solving this optimization problem is a simple adaption
of Proposition \ref{prop_rpls_sol}; the penalized regression
problem is replaced by a penalized non-negative regression problem.
For many penalty types, these problems have a simple solution.  
With the $\ell_{1}$-norm penalty, for example, 
the soft-thresholding operator in the update for $\vvec$ is replaced
by the positive soft-thresholding operator: $\vvec = P( \M \uvec,
\lambda) = ( \M \uvec - \lambda )_{+}$ \citep{allen_snn_gpca_2011}.
Our RPLS 
framework, then, gives a simple and computationally efficient method
for enforcing non-negativity in the PLS loadings.  Also, as in
\citet{allen_snn_gpca_2011}, non-negativity and quadratic operators
can be used in 
combination for PLS to create flexible approaches for 
high-dimensional data sets.


\section{Simulation Studies}
\label{section_sims}

We explore the performance of our RPLS methods for regression in a
univariate and a multivariate simulation study.

\subsection{Univariate Simulation}

In this simulation setting, we compare the mean squared prediction
error and variable selection performance of RPLS against
competing methods in the univariate regression response setting with
correlated 
predictors. Following the approach in Section 5.3 of
\citet{chun_sparse_pls_2010}, 
we include scenarios where $n$ is greater than $p$ and where $n$ is
less than $p$ with differing levels of noise. For the $n > p$ setting,
we use $n = 400$ and $p = 40$; for the $n < p$ setting, we use $n =
40$ and $p = 80$. In each case, 75\% of the $p$ predictors are true
predictors, while the remaining 25\% are spurious predictors that are
not used in the generation of the response. For the low and high noise
scenarios, we use signal-to-nose ratios (SNR) of 10 and 5.

To create correlated predictors as in \citet{chun_sparse_pls_2010}, we
construct hidden variables 
$H_1,\ldots,H_3$, where $H_i \sim \mathcal{N}(0,
  25\mathbf{I}_n)$. The columns of the predictor matrix $X_i$ 
  are generated as the sum of a hidden variable and independent random
  noise as follows: $X_i = H_1 + \varepsilon_i$ for $1 \leq i
  \leq 3p/8$, $X_i = H_2 + \varepsilon_i$ for $3p/8 < i \leq 3p/4$, and
$X_i = H_3 + \varepsilon_i$ for $3p/4 < i \leq p$, where $\varepsilon_i \sim
\mathcal{N}(0, \mathbf{I}_n)$. The response vector $Y = 3H_1 - 4H_2 +
f$, where $f \sim \mathcal{N}(0, 25\mathbf{I}_n /
\text{SNR})$. Training and test sets for all settings of $n$, $p$ and
SNR are created using this approach.

For the comparison of methods, $\mathbf{X}$ and $Y$ are standardized,
and parameter selection is carried out using 10-fold cross validation
on the training data. For the sparse partial least squares (SPLS) method
described in \citet{chun_sparse_pls_2010}, the \verb|spls| R package
 \citep{spls_package} is used with 
$\eta$ chosen from the sequence $(0.1, 0.2, \ldots, 0.9)$ and $K$ from
5 to 10. Note that for our methods, we choose to select $K$
automatically via the lasso penalized PLS regression problem described
in Section \ref{section_pred} with penalty parameter $\gamma$. 
Thus for RPLS, lasso penalties were used with $\lambda$ and $\gamma$
 chosen from 25 equally spaced values between $10^{-5}$ and
$\log(\max(|\mathbf{X}'Y|))$ on the log scale. For the lasso and
elastic net, the \verb|glmnet| R package \citep{glmnet} is used with
the same choices 
for $\lambda$.

The average mean squared prediction error (MSPE), true positive rate (TPR),
and false positive rate (FPR) across 30 simulation runs are given in
Table \ref{univariate_sim}. The penalized regression methods clearly
outperform traditional PLS in terms of the mean squared prediction
error, with RPLS having the best prediction accuracy among all methods. SPLS and RPLS are nearly perfect in correctly identifying the
true variables, but SPLS tends to have higher rates of false
positives. In contrast, the lasso and elastic net have high
specificity, but fail to identify many true predictors.

\begin{table}[ht]
\begin{center}
\scalebox{.9}{
\hspace{-.65in} \begin{tabularx}{\textwidth}{lrrrXlrrr}
\multicolumn{4}{l}{\textbf{Simulation 1}: n = 400, p = 40, SNR =
  10} & & \multicolumn{4}{l}{\textbf{Simulation 2}: n = 400, p = 40, SNR =
  5}\\
Method & MSPE (SE) & TPR (SE) & FPR (SE) & & Method & MSPE (SE) & TPR
(SE) & FPR (SE)\\
\cline{1-4}\cline{6-9}
PLS & 504.2  & & & & PLS & 655.2  & & \\
    & (293.8) & & & &    & (212.9) & & \\
Sparse PLS & 72.6  & 1.00  & 0.61 & & Sparse PLS &
143.7  & 1.00 & 0.66\\
 & (4.1) & (0.00) &  (0.27) & & & (9.8) & (0.00) & (0.29)\\
RPLS & 66.4 & 1.00 & 0.22 & & RPLS & 131.4 & 1.00 & 0.19 \\
& (3.8) & (0.00) & (0.35) & & & (9.3) & (0.00) & (0.37) \\
Lasso & 70.9 & 0.60 & 0.00 & & Lasso & 139.3 & 0.49 & 0.00 \\
& (4.9) & (0.07) & (0.02) & & & (9.5) & (0.07) & (0.00) \\
Elastic net & 70.5 & 0.61 & 0.01 & & Elastic net & 139.0 & 0.50 & 0.00
\\
& (4.5) & (0.07) & (0.03) & & & (9.5) & (0.07) & (0.00)\\
\multicolumn{9}{c}{}\\
\multicolumn{4}{l}{\textbf{Simulation 3}: n = 40, p = 80,  SNR =
  10} & & \multicolumn{4}{l}{\textbf{Simulation 4}: n = 40, p = 80, SNR =
  5}\\
Method & MSPE (SE) & TPR (SE) & FPR (SE) & & Method & MSPE (SE) & TPR
(SE) & FPR (SE)\\
\cline{1-4}\cline{6-9}
PLS & 624.1  & & & & PLS & 612.6  & & \\
    & (256.5) & & & &  & (256.8) & & \\
Sparse PLS & 104.9  & 0.99  & 0.77 & & Sparse PLS &
206.4  & 0.98 & 0.70\\
 & (26.3) & (0.05) &  (0.30) & & & (53.9) & (0.07) & (0.31)\\
RPLS & 76.0 & 1.00 & 0.45 & & RPLS & 155.1 & 1.00 & 0.52 \\
& (20.8) & (0.00) & (0.43) & & & (59.0) & (0.00) & (0.43) \\
Lasso & 83.7 & 0.17 & 0.02 & & Lasso & 178.3 & 0.12 & 0.01 \\
& (19.7) & (0.04) & (0.06) & & & (49.7) & (0.04) & (0.02) \\
Elastic net & 82.4 & 0.17 & 0.02 & & Elastic net & 172.7 & 0.12 & 0.01
\\
& (18.6) & (0.03) & (0.04) & & & (46.0) & (0.04) & (0.03)\\
\end{tabularx}
}
\end{center}
\caption{\em Comparison of mean squared prediction error (MSPE), true
  positive rate(TPR) and false positive rate (FPR) with standard
  errors (SE).}
\label{univariate_sim}
\end{table}

\subsection{Multivariate Simulation}

In this simulation setting, we compare the mean squared prediction
error of regularized PLS against competing methods for multivariate
regression. As in the univariate simulation, we include scenarios
where $n > p$ and $n < p$ with varying levels of noise, but now our
response $\mathbf{Y}$ is a matrix of dimension $n \times q$ with $q =
10$. For the 
$n > p$ scenario, we use $n = 400$ and $p = 40$ with 5 true
predictors. For the $n < p$ scenario, we use $n = 40$ and $p = 80$
with 10 true predictors. In each case, we test the methods using
signal to noise ratios (SNR) of 2 and 1.

The simulated data is generated using 8 binary hidden variables
$H_1, \ldots, H_8$ with entries drawn from the Bernoulli(0.5)
distribution. The coefficient matrix $\mathbf{A}$ contains standard
normal random entries for the first $p_{\text{true}}$ columns, with
the remaining columns set to 0. The predictor matrix $\mathbf{X} =
\mathbf{H}\cdot\mathbf{A} + \mathbf{E}$, where the entries of
$\mathbf{E}$ are drawn from the $\mathcal{N}(0, 0.1^2)$
distribution. The coefficient matrix $\mathbf{B}$ contains entries
drawn from the $\mathcal{N}(0, \text{SNR}\cdot n\cdot q /
\text{tr}(\mathbf{H}\mathbf{H}'))$ distribution. The response matrix
$\mathbf{Y} = \mathbf{H}\cdot \mathbf{B} + \mathbf{F}$, where the
entries of $\mathbf{F}$ are drawn from the standard normal
distribution. Both training and test sets are generated using this
procedure, and both $\mathbf{X}$ and $\mathbf{Y}$ are
standardized.

For the penalized methods including sparse PCA (SPCA) and regularized PLS
(RPLS) the penalty parameter $\lambda$ is chosen
from 25 equally spaced values between $-5$ and
$\log(\max(|\mathbf{X}'\mathbf{Y}|))$ on the log scale using the BIC
criterion. 
 For RPLS,
$\gamma$, the PLS regression penalty parameter for selecting $K$, is chosen
 from the same set of options as 
$\lambda$ using the BIC criterion. To obtain the coefficient
$\mathbf{\beta} = \mathbf{V}\mathbf{Z}'\mathbf{Y}_{\text{training}}$,
the columns of $\mathbf{V}$ and $\mathbf{Z}$ were normalized.
The results shown in Table \ref{multivariate_sim} demonstrate that
regularized PLS outperforms both sparse PCA and standard PLS.

\begin{table}[ht]
\begin{center}
\begin{tabular}{lllllll}
\multicolumn{3}{l}{\textbf{Simulation 1}: n = 400, p = 40, SNR =
  2} & & \multicolumn{3}{l}{\textbf{Simulation 2}: n = 400, p = 40, SNR =
  1}\\
Method \hspace{.5in} & MSPE (SE) & && Method \hspace{.5in} & MSPE (SE)
  &\\
\cline{1-3}
\cline{4-7}
SPCA & 2376.2 (337) && & SPCA & 2204.1 (313) & \\
PLS & 2567.7 (316) && & PLS & 2343.3  (281) &\\
RPLS & 404.7 (96) && & RPLS & 339.4 (175) &\\
\multicolumn{7}{c}{}\\
\multicolumn{3}{l}{\textbf{Simulation 3}: n = 40, p = 80,  SNR =
  2} & & \multicolumn{3}{l}{\textbf{Simulation 4}: n = 40, p = 80, SNR =
  1}\\
Method \hspace{.5in} & MSPE (SE) && & Method \hspace{.5in}& MSPE (SE) & \\
\cline{1-3} 
\cline{4-7}
SPCA & 711.3 (107) && & SPCA & 721.7 (109) & \\
PLS & 647.0  (101) && & PLS & 659.9 (81) &\\
RPLS &  142.0 (3) && & RPLS & 133.0 (3) &\\
\end{tabular}
\caption{\em Comparison of mean squared prediction error (MSPE) with
  standard errors (SE) for
  multivariate methods.}
\label{multivariate_sim}
\end{center}
\end{table}

\section{Case Study: NMR Spectroscopy}
\label{section_nmr}


\begin{table}[!!t]
\begin{tabular}{l|r|r|}
& Training Error & Leave-one-out CV Error \\
\hline
PCA + LDA & 0.1167 & 0.1852 \\
PLS + LDA (de Jung, 1993) & 0.0000 & 0.1481 \\
GPCA + LDA (Allen {\it et. al}, 2011) & 0.1833 & 0.1481 \\
GPLS + LDA & 0.0000 & 0.1111 \\
SPCA + LDA (Shen \& Huang, 2008) & 0.1167 & 0.1481 \\
SPLS + LDA (Chun \& Keles, 2010) & 0.0000  & 0.1111 \\
SGPCA + LDA (Allen \& Maleti{\'c}-Savati{\'c}, 2011) & 0.1833 & 0.1481 \\
SGPLS + LDA & {\bf 0.0000} & {\bf 0.0741} \\
\end{tabular}
\caption{\em Misclassification errors for methods applied to the
  neural cell NMR data.  Various methods were used to first reduce the
  dimension with the resulting factors used as predictors in linear
  discriminant analysis.}
\label{tab_class}
\end{table}

\begin{table}
\begin{center}
\begin{tabular}{l|r}
& Time in Seconds \\
\hline
Sparse PLS (via RPLS) & 1.01 \\
Sparse PLS ({\tt R} package {\tt spls}) & 1033.86 \\
Sparse Non-negative GPLS  & 28.16 
\end{tabular}
\caption{\em Timings Comparisons.  Time in seconds to compute the
  entire solution path for the neural cell NMR data.}
\end{center}
\label{tab_time}
\end{table}

We evaluate the utility of our methods through a case study on NMR
spectroscopy data, a classic application of PLS methods from the
chemometrics literature.  
We apply our RPLS methods to an {\it in vitro}
one-dimensional NMR data set  
consisting of 27 samples from five classes of
neural cell types: neurons, neural stem cells, microglia, astrocytes,
and ogliodendrocytes \citep{manganas_2007}, 
analyzed by some of the same authors using PCA methods in
\citet{allen_snn_gpca_2011}.  Data is 
pre-processed in the manner described in \citet{dunn_2005b}:  functional 
spectra is discretized into bins of size 0.04
ppms yielding a total of 2394 variables,  spectra for each sample
are baseline corrected and 
normalized to their integral, and variables are standardized.
 For all PLS methods, 
the response, $\Y$ is $27 \times 5$ and coded with indicators
inversely proportional to the sample size in each class as described
in \citet{barker_pls_2003}. 
For each 
method, five PLS or PCA factors were taken and used as predictors in
linear discriminant analysis to classify the NMR samples.  To be
consistent, the BIC method was used to select any penalty parameters
except for the Sparse PLS method of \citet{chun_sparse_pls_2010} where
the default in the 
{\tt R} package {\tt spls} was employed \citep{spls_package}.  The
Sparse GPCA and 
Sparse GPLS methods were applied with non-negativity constraints as
described in \citet{allen_snn_gpca_2011} and in Section
\ref{section_ext}.  Finally, for the 
generalized methods, the quadratic operator was selected by maximizing
the variance explained by the first component; a weighted
Laplacian matrix with weights inversely proportional to the
Epanechnikov kernel with a bandwidth of 0.2 ppms was employed
\citep{allen_gmd_2011}.

\begin{figure}[!!t]
\includegraphics[width=7in,clip=true,trim=0in .75in 0in 0in]{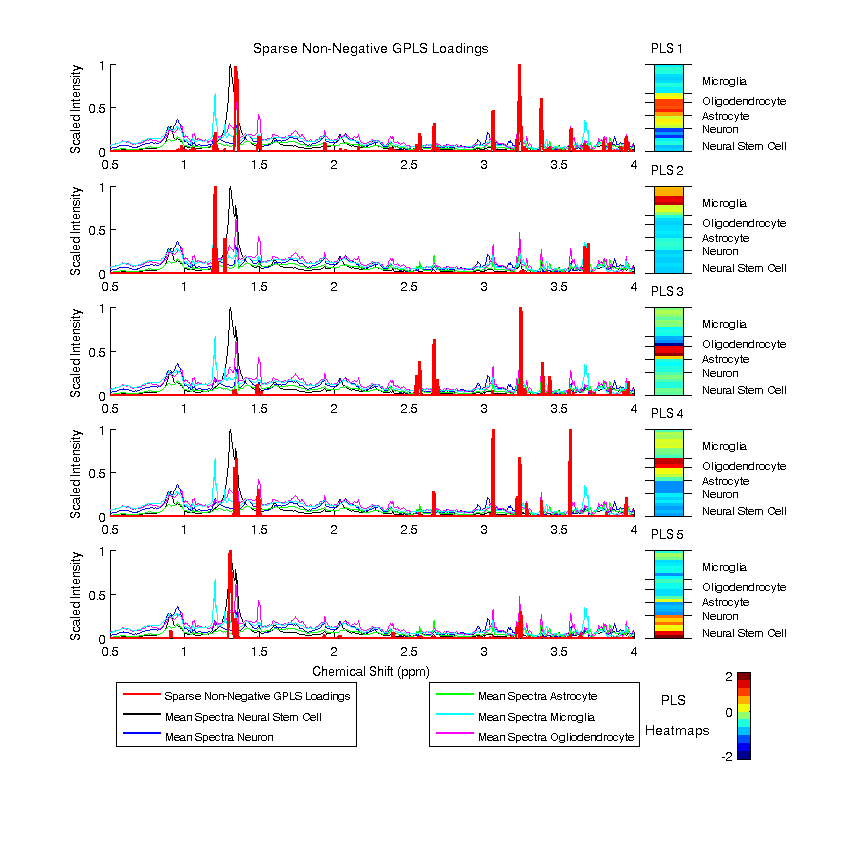}
\caption{\em Sparse Non-negative Generalized PLS loadings and sample
  PLS heatmaps for the neural cell NMR data.  The loadings are
  superimposed on the mean scaled 
  spectral intensities for each class of neural cells.}
\label{fig_nmr_ours}
\end{figure}

\begin{figure}[!!t]
\includegraphics[width=7in,clip=true,trim=0in .75in 0in 0in]{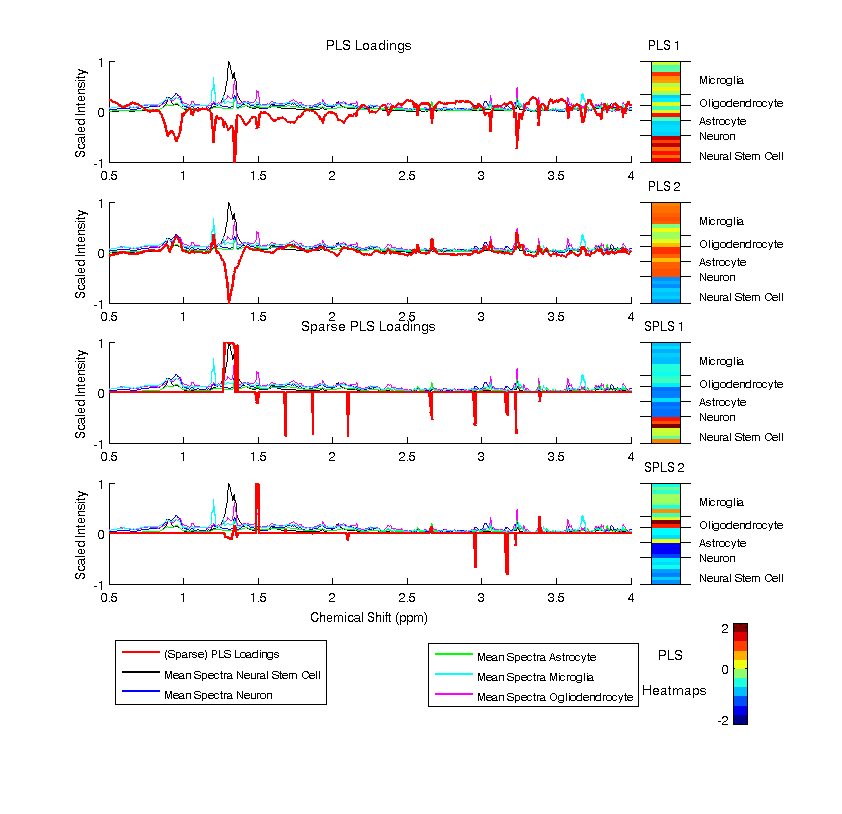}
\caption{\em The first two PLS and Sparse PLS
  \citep{chun_sparse_pls_2010} loadings and sample 
  PLS heatmaps for the neural cell NMR data.  The loadings are
  superimposed on the mean scaled 
  spectral intensities for each class of neural cells.}
\label{fig_nmr_others}
\end{figure}

In Table ~\ref{tab_class}, we give the training and leave-one out
cross-validation misclassification errors for our methods and
competing methods on the neural cell NMR data.   Notice
that our Sparse GPLS method yields the best error rates followed by
the Sparse PLS \citep{chun_sparse_pls_2010} and our GPLS methods.
Additionally, our 
Sparse GPLS methods are significantly faster than competing
approaches.  In Table ~\ref{tab_time}, the time in seconds to compute
the entire solution path (51 values of $\lambda$) is reported.  Timing
comparisons were done on a Intel Xeon X5680 3.33Ghz processor using
single-threaded scripts.

In addition to faster computation and better classification rates, our
method's flexibility leads to easily interpretable results.  We
present the Sparse GPLS loadings superimposed on the scaled spectra
from each neural cell type and sample heatmaps in Figure
\ref{fig_nmr_ours}.  For comparison, we give the first two PLS loadings in
Figure \ref{fig_nmr_others} for PLS and Sparse PLS
\citep{chun_sparse_pls_2010}. The PLS loadings are noisy, and
the sample PLS components for PLS and Sparse PLS are difficult to
interpret as the loadings are both positive and negative.  By 
constraining the PLS loadings to be non-negative, the chemical shifts
the metabolites indicative of each neural cell type are 
readily apparent with the Sparse Non-negative GPLS loadings.
Additionally as shown in the sample PLS heatmaps, the neural cell types are
well differentiated.  For example, chemical resonances at 1.30ppms ad
3.23ppms characterize Glia (Astrocytes and Ogliodendrocytes) and
Neurons (PLS 1),
resonances at 1.19ppms and
3.66ppms characterize Microglia (PLS 2), resonances at 3.23ppms
2.65ppms characterize Astrocytes (PLS 3), resonances at 1.30ppms,
3.02ppms, and 3.55ppms characterize Ogliodendrocytes (PLS 4), and
resonances at 1.28ppms and 3.23ppms characterize Neural stem cells.
Note that some of these metabolites were identified by some of the
same authors using PCA methods in \citet{manganas_2007,
  allen_snn_gpca_2011}.  Using our flexible PLS 
approach for supervised dimension reduction, 
however, gives a much clear metabolic signature of each neural cell
type.  Overall, this case study on NMR spectroscopy data has revealed
the many strengths of our method as well as identified possible
metabolite biomarkers for further biological investigation.

\section{Discussion}
\label{section_dis}

We have presented a framework for regularizing partial least squares
with convex and order one penalties.  Additionally, we have shown how
this approach can be extend for structured data via Generalized PLS
and RPLS and extended to incorporate non-negative PLS or RPLS
loadings.  Our approaches directly solve penalized relaxations of the
SIMPLS optimization criterion.  These in turn, have many advantages
including computational efficiency, flexible modeling, easy
interpretation and visualization, better feature selection, and
improved predictive 
accuracy as demonstrated in our simulations and case study on NMR
spectroscopy.

There are many future areas of research related to our methodology.
While we have briefly discussed the use of our methods for general
regression or classification procedures, specific investigation of the
RPLS factors as predictors in the generalized linear model framework
\citep{marx_pls_glm_1996, chung_spls_class_2010}, the survival
analysis framework \citep{nguyen_sruv_2002}, and others are 
needed.  Additionally, following the close connection of PLS for
classification with the classes coded as dummy variables to Fisher's
discriminant analysis \citep{barker_pls_2003}, our RPLS approach may
give an 
alternative strategy for regularized linear discriminant analysis.
Further development of our novel extensions for GPLS and Non-negative
PLS is also needed.  Finally, \citet{nadler_2005} and
\citet{chun_sparse_pls_2010} have shown 
asymptotic inconsistency of PLS regression methods when the number of
variables is permitted to grow faster than the sample size.  For
related PCA methods, a few have shown consistency of Sparse PCA in
these settings \citep{johnstone_jasa_2009, amini_2009}.  Proving
consistent recovery of the RPLS 
loadings and the corresponding regression or classification coefficients is an
open area of future research.

Finally, we have demonstrated the utility of our methods through a
case study on NMR spectroscopy data, but there are many other
potential applications of our technology.  These include 
chemometrics, proteomics, metabolomics, 
high-throughput 
genomics, imaging, hyperspectral imaging and neuroimaging.  Overall,
we have presented a flexible and 
powerful tool for supervised dimension reduction of high-dimensional
data with many advantages and potential areas of future research and
application.  
An \verb|R| package and a \verb|Matlab| toolbox named \verb|RPLS| that
implements our methods will be made publicly available.


\section{Acknowledgments}

The authors would like to thank Frederick Campbell and Han Yu for
assistance with 
the software development for this paper.  C. Peterson acknowledges
support from the Keck Center of the Gulf Coast Consortia, on the NLM
Training Program in Biomedical Informatics, National Library of
Medicine (NLM) T15LM007093; M. Vannucci and M. Maleti{\'c}-Savati{\'c}
are partially supported by the Collaborative Research Fund from the
Virgina and L. E. Simmons Family Foundation.

\appendix

\section{Proofs}

\begin{proof}[Proof of Proposition \ref{prop_rpls_sol}]
{\footnotesize
The proof of this result follows from an argument in
\citet{allen_gmd_2011}, but we outline this here for completion.  
The updates for $\uvec$ are straightforward.  We show that the
sub-gradient equations of the penalized regression problem,
$\frac{1}{2} || \M \uvec - \vvec || + \lambda P(\vvec)$, for
$\vvec^{*}$ as defined in the stated result are equivalent to the KKT
conditions of \eqref{sing_fac_rpls}.   The sub-gradient equation of the
latter is, $\M \uvec - \lambda \nabla P( \vvec^{*}) - 2 \gamma^{*}
\vvec^{*} = 0$, where $\nabla P()$ is the sub-gradient of $P()$ and
$\gamma^{*}$ is the Lagrange multiplier for the inequality constraint
with complementary slackness condition, $\gamma^{*} ( (\vvec^{*})^{T}
\vvec^{*} - 1) = 0$.  The sub-gradient of the penalized regression
problem is $\M \uvec - \hat{\vvec} - \lambda \nabla P ( \hat{\vvec} )
= 0$.  Now, since $P()$ is order one, we this sub-gradient is
equivalent to $\M \uvec - \frac{1}{c} \tilde{\vvec} - \lambda \nabla P
( \tilde{\vvec} )$ for any $c > 0$ and for $\tilde{\vvec} =
c \hat{\vvec}$.  Then, taking $c = 1 / || \hat{\vvec} ||_{2} = 1/2
\gamma^{*}$ for any $\hat{\vvec} \neq 0$, we see that both the
complimentary slackness condition is 
satisfied and the sub-gradients are equivalent.  It is easy to verify
that the pair $(0,0)$ also satisfy the KKT conditions of
\eqref{sing_fac_rpls}. 
}\end{proof}

\begin{proof}[Proof of Corollary \ref{cor_rpls_sol}]
{\footnotesize
The proof of this fact follows in a straightforward manner from that
of Proposition \ref{prop_rpls_sol} as the only feasible solution for
$\uvec$ is $\uvec^{*} = 1$.  We are then left with a concave
optimization problem, $\maximize_{\vvec} \ \vvec^{T} \M - \lambda
P(\vvec) \ \textrm{subject to} \ \vvec^{T} \vvec \leq 1$.  From the
proof of Proposition \ref{prop_rpls_sol}, we have that this
optimization problem is equivalent to the desired result.  Since we
are left with a concave problem, the global optimum is achieved.
}\end{proof}

\begin{proof}[Proof of Proposition \ref{prop_g_rpls}]
{\footnotesize
First, define $\tilde{\Q}$ to be a matrix square root of $\Q$ as in
\citet{allen_gmd_2011}.  In this paper, they showed that Generalized
PCA was equivalent to PCA on the matrix $\tilde{\X} = \X \tilde{\Q}$
for projected factors $\V = \tilde{\Q}^{\dagger} \tilde{\V}$.  In
other words, if $\tilde{\X} = \tilde{\U} \tilde{\D} \tilde{\V}$ is the
singular value decomposition, then the GPCA solution, $\V$ can be
defined accordingly.  Here, we will prove that the multi-factor RPLS
problem for $\tilde{\X}$ and $\tilde{\vvec}_{k}$
is equivalent to the stated Generalized RPLS problem \eqref{g_rpls}
for $\lambda = 0$.  The constraint regions are trivially equivalent so
we must show that $\tilde{\vvec}_{k}^{T} \tilde{\Pmat}_{k-1}
\tilde{\M} = \vvec^{T} \Q \Pmat_{k-1} \Q \M$.  The PLS factors,
$\tilde{\zvec}_{k} = \tilde{\X} \tilde{\vvec}_{k} = \X \Q \vvec_{k} =
\zvec_{k}$, are equivalent.  Ignoring the normalizing term in the
denominator, the columns of the projection weighting matrix are
$\tilde{\R}_{k} = \tilde{\X} ^{T} \tilde{\zvec}_{k} = \tilde{\Q}^{T}
\X \zvec_{k} = \tilde{\Q}^{T} \R_{k}$.  Thus, the $ij^{th}$ element of
$\tilde{\R}^{T} \tilde{\R} = \zvec_{i}^{T} \X \tilde{\Q}
\tilde{\Q}^{T} \X^{T} \zvec_{j} = \R_{i}^{T} \Q \R_{j}$ as stated.
Putting these together, we have $\tilde{\vvec}_{k}^{T} \tilde{\Pmat}_{k-1}
\tilde{\M} = \vvec_{k}^{T} \tilde{\Q} ( \mathbf{I} - \tilde{\R}_{k-1}
( \R_{k-1}^{T} \Q \R_{k-1} )^{-1} \tilde{\R}_{k-1}^{T} )
\tilde{\Q}^{T} \X^{T} \Y$ which simplifies to the desired result.  

Following this, the proof of the first part is a straightforward
extension of Theorem 1 and Proposition 1 in \citet{allen_gmd_2011}.
The proof for the second part follows from combining the arguments in
Proposition \ref{prop_rpls_sol} and those in the proof of Theorem 2 in
\citet{allen_gmd_2011}.  
}\end{proof}

\singlespacing
{\small
\bibliographystyle{Chicago}
\bibliography{pls,nmr,tensors}
}

\end{document}